\documentclass[11pt,reqno,a4paper]{amsart}
\usepackage{color,graphicx,enumerate,amssymb}
\usepackage{hyperref}
\usepackage{graphicx}
\usepackage[export]{adjustbox}
\usepackage{eucal}
\usepackage[utf8]{inputenc}
\usepackage{tabularx}
\usepackage{lineno}
\usepackage{mathtools}
\usepackage{comment}
\usepackage{graphicx}

\usepackage[numbers,sort&compress]{natbib}
\usepackage{svg}
\usepackage{array}

\newcommand{\cX}{{\mathcal X}}

\newcommand{\uu}{\mathbf{ u}}

\def\XXint#1#2#3{{\setbox0=\hbox{$#1{#2#3}{\int}$ }
\vcenter{\hbox{$#2#3$ }}\kern-.6\wd0}}

\newcommand\norm[1]{\Arrowvert {#1}\Arrowvert}

\newtheorem{theorem}{Theorem}[section]

\newtheorem{definition}[theorem]{Definition}
\newtheorem{example}[theorem]{Example}

\newtheorem{lemma}[theorem]{Lemma}

\newtheorem{remark}[theorem]{Remark}

\title [Segregated  Lipschitz Learning  for Classification]{  Graph-Based Semi-Supervised Segregated  Lipschitz Learning }
\author{Farid Bozorgnia, Yassine Belkheiri, Abderrahim Elmoataz}

\address{CAMGSD, Department of Mathematics, Instituto Superior T\'{e}cnico, Lisbon, Portugal}
\email{faridb.bozorgnia@tecnico.ulisboa.pt}

\address{University of Caen Normandy, GREYC UMR CNRS 6072, France }
\email{ yassine.belkheiri@unicaen.fr}

\address{University of Caen Normandy, GREYC UMR CNRS 6072, France }
\email{ abderrahim.elmoataz@unicaen.fr}

\subjclass[2000]{}


\begin{document}
 
\maketitle

\begin{abstract}

  This paper presents an approach to semi-supervised learning for the classification of data using the  Lipschitz Learning on graphs. We develop a graph-based semi-supervised learning framework that leverages the properties of the infinity Laplacian to propagate labels in a dataset where only a few samples are labeled. By extending the theory of spatial segregation from the Laplace operator to the infinity Laplace operator, both in continuum and discrete settings, our approach provides a robust method for dealing with class imbalance, a common challenge in machine learning. Experimental validation on several benchmark datasets demonstrates that our method not only improves classification accuracy compared to existing methods but also ensures efficient label propagation in scenarios with limited labeled data.  

\end{abstract}
\bigskip

\section{Introduction}

 In recent years, nonlinear Partial Differential Equations (PDEs) on graphs have attracted increasing interest due to their natural emergence in various applications in mathematics, physics, biology, economics, and data science. For example, they are relevant in fields such as Internet and road networks, social networks, population dynamics, image processing, and machine learning. Indeed, a large amount of complex and irregular data is generated daily from various sources, including the internet, images, point clouds, 3D meshes, and biological networks. These datasets can be directly represented or modeled as graphs or functions defined on graphs.

Consequently, intensive research aims to develop new methods for processing and analyzing data defined on graphs and adapt classical signal and image processing methods and concepts to graphs.

Recently, several works have focused on the study of PDEs on graphs and their local or non-local continuous limits in the Euclidean domain. Among the significant contributions, we can mention the works of teams such as those of A. Bertozzi \cite{bertozzi2012, merkurjev2013, garcia2014}, Y. VanGuennip \cite{vanguennip2014, vanguennip2018, budd2021}, S. Osher \cite{osher2019, gilboa2008}, and J. Mazon \cite{mazon2022, mazon2023, mazon2023book}, who have proposed to adapt several types of continuous PDEs and variational models to graphs. Other research aims to transpose various continuous PDEs to graphs, such as \(p\)-Laplacian equations \cite{alaoui2016, elmoataz2015p, hafiene2018nonlocal}, infinity Laplacian \cite{elmoataz2015p}, "game" \(p\)-Laplacian \cite{elmoataz2017game}, Hamilton-Jacobi equations with or without diffusion \cite{fadili2022, ennaji2021, A12}, related to certain stochastic games \cite{elmoataz2017connection, elmoataz2017nonlocal}, mean curvature flow equation, or \(p\)-biharmonic equation \cite{fadili2022}.

\label{sec:1}

In this paper, utilizing the characteristic of Infinity Laplacian in propagating,  we propose a novel graph-based semi-supervised learning method aimed at classifying large volumes of unlabeled data, particularly in scenarios involving imbalanced datasets and limited labeled samples. Our approach exploits the geometric and topological characteristics of the unlabeled data, integrating these properties to enhance the development and performance of various algorithms.



\section{Mathematical preliminaries on Infinity Laplacian }

In this section,  we provide an overview of the infinity Laplacian. The infinity Laplacian equation lies at the crossroads of several mathematical disciplines and is applied in various fields, including optimal transportation, game theory, image processing, computer vision, and surface reconstruction.  It was first explored by Gunnar Aronsson \cite{aronsson1965minimization}, who, motivated by classical analysis, sought to develop Lipschitz extensions of functions. Over the past decade, researchers in PDEs have made significant strides in establishing the existence, uniqueness, and regularity of solutions to this equation. For detailed discussions on the uniqueness of Lipschitz extensions and the theory of absolute minimizers, see \cite{aronsson2004tour, crandall2001optimal}. For recent advancements in the numerical approximation of eigenvalues and eigenfunctions of the Infinity Laplace operator, we refer to \cite{Bozorgnia2024}. Additionally, connections have emerged linking this equation (and the $p$-Laplacian) to continuous values in Tug-of-War games. For further insights and applications related to these equations, refer to \cite{aronsson1967extension,aronsson2004tour, crandall2008visit, crandall2007uniqueness, juutinen2006evolution, juutinen1999eigenvalue, manfredi2012definition, manfredi2012dynamic} and the associated references.

\subsection{Local continuous infinity Laplacian}
The infinity Laplacian operator is defined as follows:
\begin{equation}
\label{eq:infLapOp}
  \Delta_{\infty} u \ = \  (\nabla u)^T D^2u \nabla u  \ = \ \sum_{i,j=1}^{d} \frac{\partial u}{\partial x_i} \frac{\partial u}{\partial x_j}\frac{\partial^2 u}{\partial x_i \partial x_j}.
\end{equation}
The operator in (\ref{eq:infLapOp}) can be normalized by $\frac{1}{|\nabla u|^2}$, e.g., cf. \cite{barron2008infinity}.

A function $u$ is said to be infinity harmonic if it solves the homogeneous infinity Laplacian equation
\begin{equation}
\label{eq:infLapEq}
 \Delta_{\infty} u \ = \ 0.
\end{equation}
in the viscosity sense.  As   mentioned  earlier this equation can be derived as the limit of a sequence of $p$-Laplacian equations 
\[
\Delta_p u = \textrm{div}(|\nabla u|^{p-2}\nabla u) = 0, 
\]
under certain boundary conditions for $p \rightarrow \infty$.

In \cite{Leon2}, the author examines the convergence rate of solutions to the $p$-Laplace equation as they approach solutions of the infinity-Laplace equation. Specifically, the study establishes a convergence rate for solutions to the $p$-Laplace equation as 
$p$ tends to infinity. The main results indicate that this convergence occurs at a rate proportional to 
$p^{\frac{-1}{4}}$   in general cases, and improves to a faster rate of 
$p^{\frac{-1}{2}}$  under certain conditions, such as the presence of a positive gradient.  

The infinity Laplace equation is closely connected to the problem of finding the Absolute Minimal Lipschitz Extension (AMLE) \cite{aronsson1965minimization, aronsson2004tour}. This framework aims to identify a continuous real-valued function that maintains the smallest possible Lipschitz constant on any open set compactly contained within  $\Omega$. This perspective is essential for developing numerical approximation schemes for solving the infinity Laplace equation (\ref{eq:infLapEq}).

We consider two metric spaces $(X, dX)$ and  $(Y,dY )$, which have the isometric extension property.

\begin{definition} Let $\Omega$  be a subset of  $X$ and $f: \Omega\rightarrow Y$ be a Lipschitz function. If  $g$
extends $f$  and  $Lip(g, X) = Lip(f,\Omega)$ then we say that g is a minimal Lipschitz extension (MLE) of  $f$.
\end{definition}

A function $u\in W^{1,\infty}(\Omega)$ is called absolutely minimizing Lipschitz extension of a Lipschitz function $g$ if $u\vert_{\partial\Omega}=g$ and
\[
\norm{\nabla u}_{L^\infty(\Omega')}\leq\norm{\nabla v}_{L^\infty(\Omega')},
\]
for all open sets $\Omega'\subset\Omega$ and all $v$ such that $u-v\in W^{1,\infty}_0(\Omega')$.
The connection between an AMLE and the infinity Laplacian is such that a function  $u\in\mathrm{Lip}(\Omega)$ is an AMLE of a Lipschitz function $g:\partial\Omega \rightarrow \mathbb{R}$ if and only if $u$ is a viscosity solution of the infinity Laplacian equation with $u=g$ on~$\partial\Omega$.

Implementing the infinity Laplacian and related equations often involves significant mathematical challenges, such as the lack of regularity and difficulties in establishing the existence and uniqueness of solutions. From a numerical perspective, these equations can behave quite differently from their linear counterparts, leading to interesting and complex theoretical problems.

 To illustrate the motivation for using the nonlinear Infinity Laplacian operator in semi-supervised learning tasks, let us examine the following example. Let \(\Omega = [-1, 1] \times [-1, 1]\), and let \(u\) be the solution of the problem:

\begin{equation}
\label{eq1}
\left\{
\begin{array}{ll}
  \Delta u = 0  & \text{in} \ \Omega \setminus \{(0, 0)\},\\
  u(0, 0) = 1, \ u = 0  & \text{on} \ \partial \Omega.
\end{array}
\right.
\end{equation}

Next, we replace the Laplace operator with the Infinity Laplacian and solve the corresponding problem. The results are illustrated in Figure \ref{fig:1}. The figure highlights that the standard Laplacian is degenerate in this setting, which leads to less effective label propagation. In contrast, the Infinity Laplacian offers better propagation. 

To demonstrate how the Infinity Laplacian handles imbalanced data more effectively, we solve the problem with the Infinity Laplacian using 10 interior points with a value of +1 and the middle point assigned a value of -1. Figure \ref{fig:2} shows the resulting surface of the solution across the grid, indicating how the Infinity Laplacian distributes the labels more efficiently in such scenarios.

\begin{figure}[h!]
\includegraphics[scale=0.25]{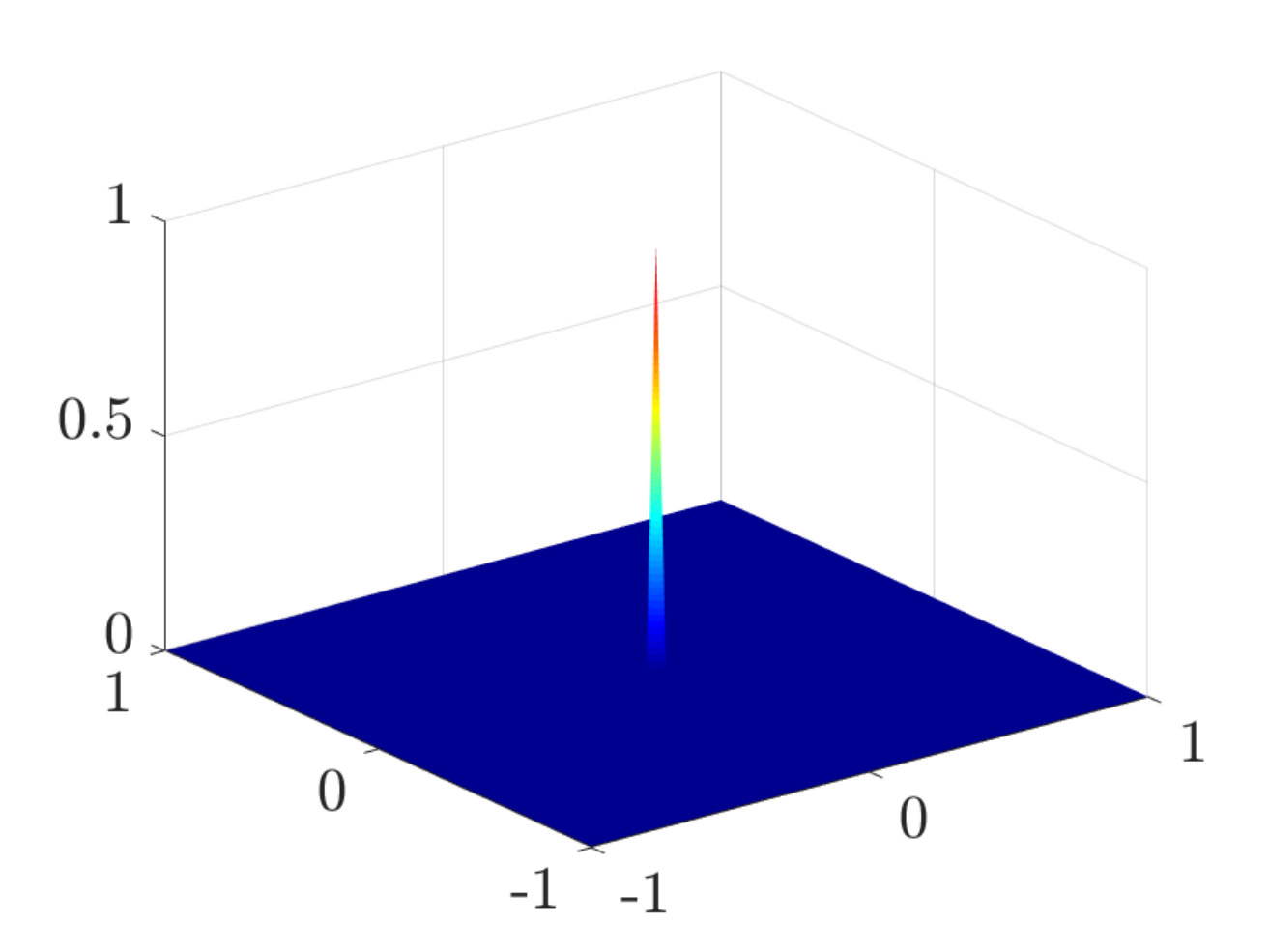}
\includegraphics[scale=0.3]{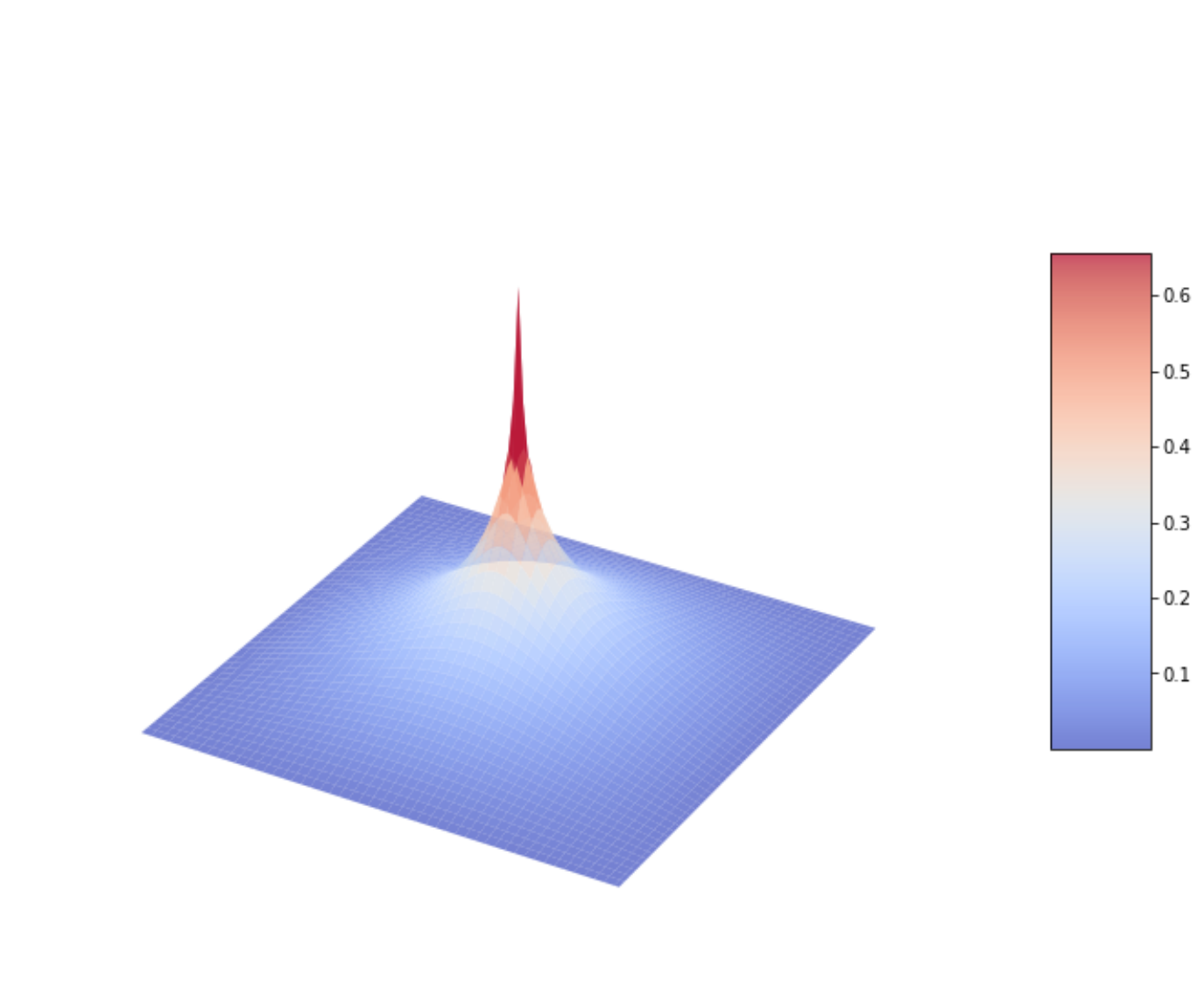}
\includegraphics[scale=0.2]{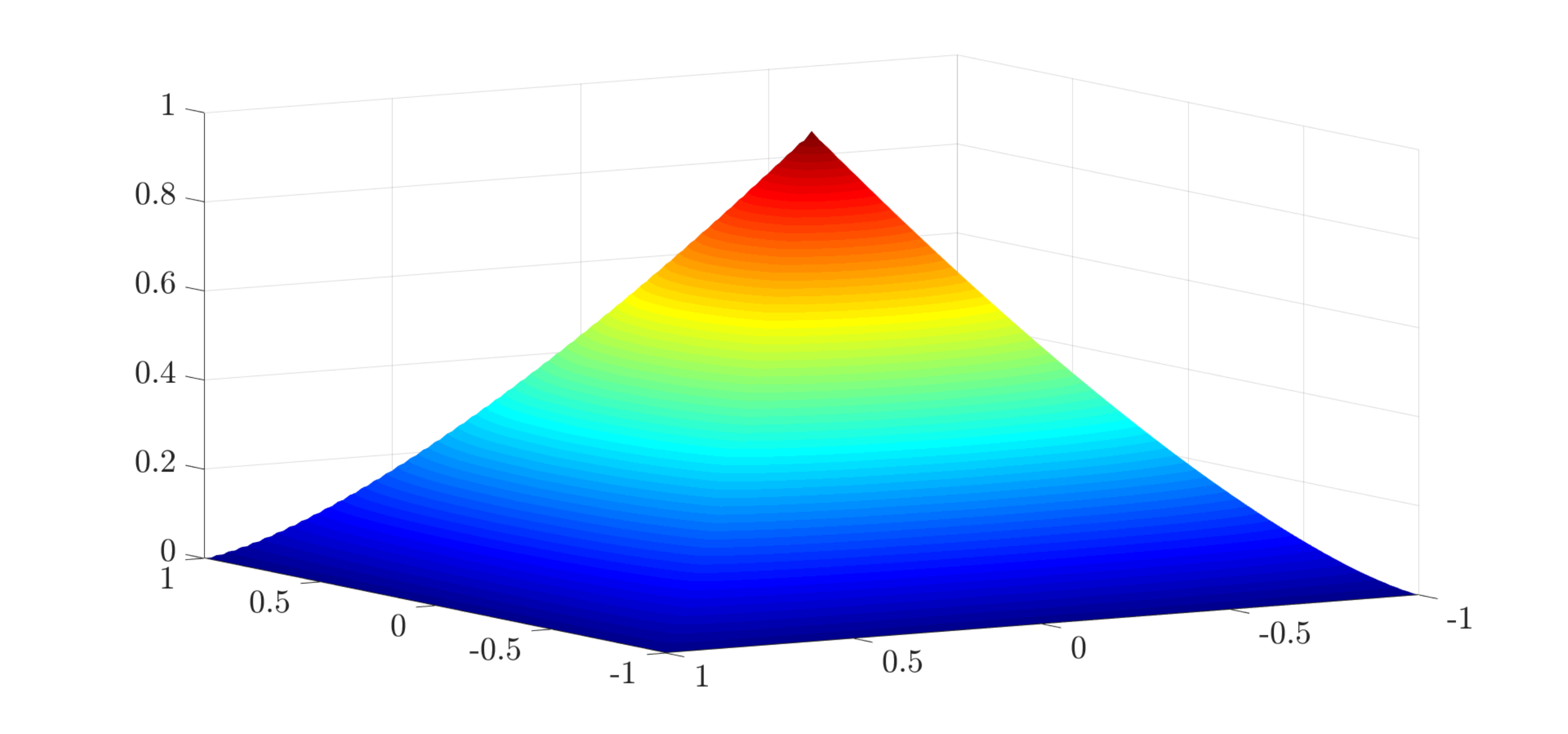}
\caption{The first picture on top-left shows the initial value, the right indicates the solution of Problem \refeq{eq1} and the last picture is the solution of infinity Laplacian with the same boundary condition.}
\label{fig:1}       
\end{figure}
 
\begin{figure}[h!]
\includegraphics[scale=0.35]{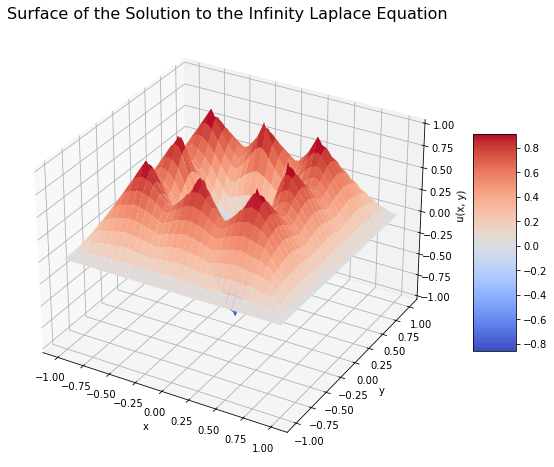}
\includegraphics[scale=0.35]{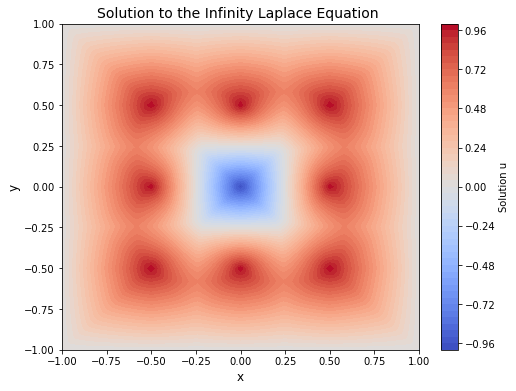 }
\caption{.}
\label{fig:2}       
\end{figure}
We would like to highlight the results from \cite{alaoui2016}, where $p$-Laplacian regularization for semi-supervised learning (SSL) on large, geometric random graphs is studied. The key problem addressed in this work is how to utilize a small set of labeled points on a graph to predict the labels of the remaining unlabeled points by leveraging the smoothness properties of the data.  The authors study the asymptotic behavior of the regularization technique as the number of unlabeled points grows indefinitely.  When $p\le  d$, the solution becomes degenerate and discontinuous, essentially leading to spikes and poor predictions.  For  $p \ge d + 1$, the solution becomes smooth and well-behaved, making it ideal for prediction tasks. For small values of p (closer to 2), the function estimate is highly sensitive to the underlying data distribution. This means the estimate is informed by the density of the unlabeled data, which can improve predictions if the data aligns with the cluster assumption. However, it risks overfitting or discontinuities, especially in higher dimensions. For large values of $p$ (approaching infinity), the estimate tends to be a smooth, Lipschitz continuous solution that is independent of the data distribution. This happens because large values of $p$ lead to the  Absolute  Minimal Lipschitz Extension (AMLE). While this solution is smooth, it can ignore valuable information from the unlabeled data.   
\subsection{Non-local continuous infinity Laplacian}
Chambolle et al. \cite{chambolle2012holder} introduced a Hölder-type infinity Laplacian, which can be viewed as a non-local variant of the traditional infinity Laplacian. Their approach focuses on minimizing the following non-local functional:

\begin{equation}
\label{eq_nonlocal1}
\int_{\Omega \times \Omega} \frac{|u(x_i)-u(x_j)|^{p}}{|x_i-x_j|^{\alpha p}} d x_i d x_j, \quad \text { for } \alpha \in[0,1]
\end{equation}

The corresponding Euler-Lagrange equation for this functional is given by

\begin{equation}
\label{eq_nonlocal2}
\int_{\Omega} \frac{|u(x_i)-u(x_j)|^{p-1}}{|x_i-x_j|^{\alpha}} \frac{\operatorname{sign}(u(x_i)-u(x_j))}{|x_i-x_j|^{\alpha}} d x_j=0 . 
\end{equation}

As $p$ tends to infinity, this equation formally converges to a non-linear, non-local equation $L(u)=0$ on $\Omega$ with

\begin{equation}
\label{eq_nonlocal3}
L(u)=\max _{x_j \in \Omega, x_j \neq x_i}\left(\frac{u(x_j)-u(x_i)}{|x_j-x_i|^{\alpha}}\right)+\min _{x_j \in \Omega, x_j \neq x_i}\left(\frac{u(x_j)-u(x_i)}{|x_j-x_i|^{\alpha}}\right), \quad \text { for } x_i \in \Omega . 
\end{equation}

This operator $L(u)$ is referred to as the Hölder infinity Laplacian.
\[
\]
 
\section{Calculus on Graphs}
\label{sec:2}
In fully supervised learning, a model is trained using labeled data to learn a generalized pattern. In semi-supervised learning, we use both labeled and unlabeled data, taking advantage of the large amount of unlabeled samples. This approach is especially beneficial in situations where labeled data is scarce.  The lack of labeled data can limit a model's ability to learn and generalize effectively. However, by incorporating a larger set of unlabeled samples, the model can uncover additional insights from the data, enhancing performance on tasks with few labeled examples.

Graph-based semi-supervised learning methods are a common and important class of semi-supervised learning techniques. These methods represent data as connectivity graphs, capturing the entire dataset's structure, including labeled and unlabeled samples. For a comprehensive review of existing methods, refer to \cite{chong2020graph}.

Let the  data set consist of $n$ samples
 $X={\{ x_1,  x_2, \cdots ,x_n }\} \subset \mathbb{R}^d$,    we assume there is a subset of the nodes $\Gamma={\{ x_1, \cdots, x_m}\}  \subset   X $  that their labels are given, where $m \ll n$. In graph-based semi-supervised learning, we aim to extend labels to the rest of the vertices $\{x_{m+1}, \cdots, x_n\}$.

A common approach in semi-supervised learning involves utilizing unlabeled data by constructing a graph across the dataset. Given a set of samples $X$, which includes a few labeled samples (without their labels $Y$) and a large amount of unlabeled data, we build a \textit{K-Nearest Neighbor} (K-NN) graph over $X$. 
 
The next step involves creating an adjacency/weight matrix, denoted by  $W$, over the constructed K-NN graph over data $X$. This matrix captures the similarities between pairs of data points. If the data set consists of $n$ samples
    then the weight matrix $W$ is an $n \times n $  symmetric matrix, where the element $w_{ij}$  represents the similarity between two samples  $x_i$  and $x_j$. The similarity is always non-negative and should be large when  $x_i$  and $x_j$ are close together spatially, and small (or zero), when $x_i$  and $x_j$  are far apart. The degree of a vertex $x_i$ is given by $d(x_i)=\sum\limits_{j\sim i} w_{ij}$.  
    
Here,  we review the definitions of $p$-Laplace and  Infinity Laplacian on the graph given in \cite{DEL, Abde1, Abde2, elmoataz2015p, elmoataz2017nonlocal}.   The weighted $p$-Laplace  operator of a function $u \in H(V)$ is defined by
\[
\Delta_{w,p} u(x_i)=\sum_{j\sim i} w_{ij}^{\frac{p}{2}}|u(x_i)-u(x_j)|^{p-2} (u(x_i)-u(x_j)).
\]
 The $p$-Laplace learning   for  large values of $p > d,$( where 
 $d$ is the dimension of ambient space) \cite{alaoui2016}  involves solving the nonlinear graph $p$-Laplacian equation:
\begin{equation}
\label{p-learning}
\left\{
\begin{array}{ll}
  \Delta_{w,p} u(x_i) = 0  &  x_i \in  X\setminus\Gamma,\\
    u(x_i) =  y_i   &  x_i \in  \Gamma.
\end{array}
\right.
\end{equation}
Note by the Sobolev embedding Theorem for  $ p> d$ the H\"{o}lder continuity of solution allows to define the boundary values.

The $L_p$ norm of the gradient for $u \in H(V)$ at a vertex of the graph is defined by
\[
\| \nabla_{w} u(x_i)\|_{p}=[\sum_{j\sim i} w_{ij}^{\frac{p}{2}}|u(x_i)-u(x_j)|^{p}]^{\frac{1}{p}} , \, 0<p<+\infty.
\]
By considering the upwind method to discretize the gradient on the graph we have
\[
\| \nabla_{w}^{\pm} u(x_i)\|_{p}^{p}=\sum_{j\sim i} w_{ij}^{\frac{p}{2}}((u(x_i)-u(x_j))^{\pm} )^{p} , \, 0<p<+\infty.
\]
For the $L_\infty$-norm we have
\[
\|\nabla_{w}^{\pm} u(x_i)\|_{\infty}=\max_{j \sim i}\, \left(\sqrt{w_{ij}}\, \left(u(x_i)-u(x_j)\right)^{\pm}\right).
\]
The Infinity Laplacian   is defined by 
\[
 \Delta_{w,\infty} u(x_i):= \frac{1}{2}\left[\|\nabla^{-}_{w} u(x_i)\|_{\infty}- \|\nabla^{+}_{w} u(x_i)\|_{\infty} \right].
\]


\subsection{Infinity Laplace equation}
To solve the infinity Laplacian on a graph, we use two methods. The first is an approximation of the infinity Laplacian using the Minimal Lipschitz Extension, and the second is solving the evolutionary infinity Laplace equation \cite{abderrahim2014nonlocal}.
\subsubsection{Minimal Lipshitz extension}
 As in the previous section,  for given a graph \(G = (V, E)\) with a set of vertices \(V = \{x_1, x_2, \dots, x_n\}\) and a weight matrix \(W\) representing edge weights (or similarities), the Infinity Laplacian  denoted by \(\Delta_{w,\infty} u\) or  \(\Delta_{\infty} u\) at a vertex \(x_i\) is defined as:
\[
\Delta_{w,\infty} u(x_i) = \frac{1}{2}\left[\max_{j \sim i}\left( \sqrt{w_{ij}}\max\left(u(x_i)-u(x_j), 0\right)\right) +\max_{j \sim i}\left( \sqrt{w_{ij}}\max\left(u(x_j)-u(x_i), 0\right)\right)\right],
\]
where \( j \sim i \) indicates that \(x_j\) is a neighbor of \(x_i\), and \( w_{ij} \) is the weight (or similarity) between vertices \(x_i\) and \(x_j\).

To solve the Infinity Laplace equation \( \Delta_{w,\infty} u = 0 \) on a graph,  we aim to find a function \(u: V \to \mathbb{R}\) that satisfies the equation at each vertex, subject to given boundary conditions (e.g., fixed values of \(u\) at certain labeled vertices).

Define the graph structure \(G\), including the vertices \(V\), edges \(E\), and the weight matrix \(W\). Partition the vertex set \(V\) into labeled vertices \( \Gamma \) (where labels are known) and unlabeled vertices \( V \setminus \Gamma \) (where labels are to be predicted). Assign known labels to the vertices in \( \Gamma \). Initialize values for \( u(x_i) \) at the unlabeled vertices \( x_i \in V \setminus \Gamma \). This could be random, zero, or based on some heuristic (such as averaging over neighbors).

Use an iterative scheme to update the values of \(u(x_i)\) for the unlabeled vertices. A common approach is to use fixed-point iteration or Gauss-Seidel iteration, where the update rule at each iteration \(k+1\) is given by minimizing the discrete Lipschitz constant at each vertex.

One defines a discrete Lipschitz constant $L(u_i)$ of $u$ in $x_i$ for $i\in\{1,\dots,n\}$ as
\[
L(u_{i})= \max_{j\in N_i} \frac{|u_i-u_j|}{d_{ij}},
\]
where $d_{ij}= \|x_i- x_j \| $ distance between node $x_i$ and $x_j$. In \cite{phan2015extensions}[Theorem 4.3.1]  it was proved that the minimizer of this discrete Lipschitz constant with respect to $u_i$ is given by
\begin{align*}
    u^*_i=\text{argmin}_{u_i} L(u_i)=\frac{d_{is}u_{r}+d_{ir}u_{s}}{d_{ir}+d_{is}},
\end{align*}
where the indices $r,s\in N_i$ are chosen such that
\begin{align}\label{eq:idx_max_inf_lapl}
(r,s)\in\arg\max_{k,l\in N_i}{\left\{ \frac{|u_{k}-u_{l}|}{d_{ik}+d_{il}}\right\}}.    
\end{align}
Thus to update the value of \( u(x_i) \)  we use  the following iterative scheme:
\begin{equation} \label{L}
u^{*(k+1)}_{i} = \frac{d_{is} u_r^{(k)} + d_{ir} u_s^{(k)}}{d_{ir} + d_{is}}.
\end{equation} 
 We repeat the iterative update process until the changes in \( u(x_i) \) between successive iterations fall below a predefined tolerance level.

\subsubsection{Infinity Laplacian \cite{abderrahim2014nonlocal}}
In this part, we present an alternative approach to solve the following problem:
\begin{equation} \label{eq:1}
\left \{
    \begin{array}{l c l}
        \displaystyle \Delta_{w,\infty} u(x_i) = 0 && x \in  V\setminus \Gamma, \\
        \displaystyle u(x_i) = g(x_i) && x_i \in \Gamma.
    \end{array}
\right.
\end{equation}
To solve (\ref{eq:1}) iteratively, we make use of the associated evolution equation problem:
\begin{equation} \label{eq:evolution}
\left \{
    \begin{array}{l c l}
        \displaystyle \frac{\partial}{\partial t} u(x_i,t) = \Delta_{w,\infty} u(x_i, t) && x_i \in X\setminus \Gamma, \\
        \displaystyle u(x_i,t) = g(x_i) && x_i \in \Gamma,  \\
        \displaystyle u(x_i, t = 0) = u_{0}(x_i) && x_i \in X.
    \end{array}
\right.
\end{equation}
Then, using an explicit forward Euler time discretization: 
\[
\frac{\partial u}{\partial t} (x_i, t) = \frac{u^{n+1}(x_i) - u^{n}(x_i)}{ \Delta t} \]
with \( u^{n}(x_i) = u(x_i, n \Delta t) \), we have the following iterations:
\begin{equation} \label{eq:2}
\left \{
    \begin{array}{l c l}
        \displaystyle u^{n+1}(x_i) = u^{n}(x_i) + \Delta t \Delta_{w,\infty} u^n(x_i) && x_i \in X\setminus \Gamma,\\
        \displaystyle u^{n+1}(x_i) = g(x_i) && x_i \in \Gamma,\\
        \displaystyle u^{0}(x_i) = u_{0}(x_i) && x_i X.
    \end{array}
\right.
\end{equation}


\section{Related Works on Graph-Based Semi-Supervised Learning}

Early work on classification using graph-based semi-supervised learning traces back to Zhu et al. \cite{A19, A1} and Belkin et al. \cite{A3}. They transformed learning tasks on sets of vectors into graph-based problems by identifying vectors as vertices and constructing graphs that capture the relationships among data points. However, one limitation of this approach, as noted by Nadler et al. \cite{A18}, is that when the number of nodes increases while the number of labeled samples remains fixed, almost all values of the Laplacian minimizer converge to the mean of the labels on the unlabeled samples.

In the Laplace learning algorithm \cite{A18, A19}, the labels are extended by finding the solution  \( \uu: X \to \mathbb{R}^k \) for the following problem:
\begin{equation}
\label{eq:qq1}
\left \{
\begin{array}{ll}
 \mathcal{L} \uu(x_i)=0,  &  x_i\in X \setminus \Gamma, \\
  \uu=g,   &   \text{on } \Gamma, \\
  \end{array}
\right.
\end{equation}
Here  the unnormalized graph Laplacian \( \mathcal{L} \) of a function \( v \in \ell^2(X) \) is defined as:
\[
\mathcal{L} v(x_i):= \sum_{y \in X} w_{ij}(v(x_i) - v(x_j)), 
\]
where \( w_{ij} \) represents the similarity between data points \(x_i\) and \(x_j\)
Let \( \uu=(u_1, \dots, u_k) \) be a solution of \eqref{eq:qq1}. The label of node \( x_i \in X \setminus \Gamma \) is determined by:

\[
\underset{j \in \{1, \dots, k\}}{\textrm{arg max}} \, u_j(x_i).
\]

In \cite{A8}, a scheme called Poisson Learning was proposed. This method modifies equation \eqref{eq:qq1} by replacing the zero values on the right-hand side with the label values of the training points, thus solving the Poisson equation on the graph. The method extends labels from a discrete set \( \{x_i: i = 1, \dots, l\} \) to the rest of the graph's nodes by solving the following system:
\begin{equation}\label{eq:state13}
\left \{
\begin{array}{ll}
\mathcal{L} \uu(x_i) = y_i - \overline{y}   &  1 \le i \le l, \\
\mathcal{L} \uu(x_i) = 0  &  l + 1 \le i \le n, \\
\end{array}
\right.
\end{equation}
with the additional condition:
\[
\sum_{i=1}^{n} d(x_i) \uu(x_i) = 0,
\]
where \( \overline{y} = \frac{1}{l} \sum_{i=1}^{l} y_i \) is the average label vector.

In \cite{A10}, the game-theoretic \( p \)-Laplacian for semi-supervised learning on graphs was studied. It was shown that the approach is well-posed in the limit of finite labeled data and infinite unlabeled data. The continuum limit of graph-based semi-supervised learning using the game-theoretic \( p \)-Laplacian converges to a weighted version of the continuous \( p \)-Laplace equation. Additionally, the study demonstrated that solutions to the graph \( p \)-Laplace equation are approximately Hölder continuous with high probability.

In \cite{A12} the consistency of Lipschitz learning on graphs in the limit of infinite unlabeled data and finite labeled data has been studied.   It has been shown in the case of a random geometric graph with kernel-based weights, that  Lipschitz learning is well-posed in this limit but insensitive to the distribution of unlabeled data. Furthermore, on a random geometric graph with self-tuning weights, Lipschitz learning becomes highly sensitive to the distribution of the unlabeled data. In both cases, the results stem from showing that the sequence of learned functions converges to the viscosity solution of an Infinity Laplace-type equation and analyzing the structure of the limiting equation.

The work \cite{BCR} provides a thorough analysis of Lipschitz learning, particularly examining the convergence rates of solutions to the graph infinity Laplace equation towards the continuum case as the graph density increases. This work is significant for graph-based semi-supervised learning, where labels are propagated from a small labeled set to a larger unlabeled dataset by solving an equation on the graph resembling the infinity Laplacian. Their results also demonstrate that even in sparse graphs—such as those commonly used in semi-supervised learning—convergence to continuum-based infinity Laplace solutions (or absolutely minimizing Lipschitz extensions) is achievable under general assumptions. They use  a comparison with distance functions on the graph, which allows convergence rates even at low connectivity thresholds, making the approach relevant to practical graph applications.

For algorithms related to Lipschitz learning, see \cite{K17}. A graph-based semi-supervised learning approach using the theory of spatial segregation of competitive systems is discussed in \cite{bozrgnia2023graph}.
\section{Infinity Segregation}
 We consider the minimization of the following functional for large values of \( p \):
\begin{equation}\label{continuous-model}
\left\{
\begin{split}
&\min J(\mathbf{u}) := \frac{1}{p} \int_\Omega \sum\limits_{i=1}^k |\nabla u_i|^p \, dx, \\
&\text{subject to:}\\ 
&  u_i \geq 0, \quad \text{and} \quad u_i \cdot u_j = 0 \quad \text{in } \Omega,\\[5pt]
& u_i = g_i \quad \text{on } \partial \Omega.
\end{split}
\right.
\end{equation}

First, we study a simpler case where the number of components is \( k = 2 \). Let the pair \( (u_1, u_2) \) be the minimizer. It is easy to check that \( u = u_1 - u_2 \) satisfies the following equation:

\begin{equation}
\label{eq:qq3}
\left\{
\begin{array}{ll}
- \Delta_p u = 0 & \text{in } \Omega, \\
u = g_1 - g_2 & \text{on } \Gamma.
\end{array}
\right.
\end{equation}
This result follows from the fact that 
\[
|\nabla u|^p = |\nabla u_1|^p + |\nabla u_2|^p.
\]
Additionally, we have \( u_1 = \max(u, 0) \) and \( u_2 = \max(-u, 0) \). Next, as \( p \to \infty \) in  \ref{eq:qq3} we obtain:

\begin{equation}
\label{eq:qq4}
\left\{
\begin{array}{ll}
- \Delta_\infty u = 0 & \text{in } \Omega, \\
u = g_1 - g_2 & \text{on } \Gamma.
\end{array}
\right.
\end{equation}

This implies that for the binary classification we are solving the infinity Laplace equation with boundary labels  \( g_1, g_2 \in \{\pm 1\} \). The support of \( u_1 \) corresponds to the first class, and the support of \( u_2 \) corresponds to the second class.

In the rest, we assume that  $k>2$. We  define
\begin{align*}
    \hat{u}_{i}=u_{i}-\sum_{j\neq i}u_{j}.
\end{align*}
\begin{lemma}\label{uniqu5}
Let $U=(u_{1},\ldots,u_{k})$ be the minimizer  and $\Omega_{1},\ldots,\Omega_{k}$ be the corresponding supports then the following differential inequalities hold in $\Omega$
\begin{enumerate}
  \item $-\Delta_{\mathbf{p}}u_{i}\leq 0$
  \item $-\Delta_{\mathbf{p}}\hat{u}_{i}\geq 0$
\end{enumerate}
\end{lemma}
 \begin{proof}
We follow the proof in \cite{conti} and we  argue by contradiction, assuming the existence of an index $j$ such that the following inequality holds in a distributional sense
\begin{align*}
    -\Delta_{p} u_{j}>0
\end{align*}

This means that
\begin{align}
    \int_{\Omega}|\nabla u_{j}|^{\mathbf{p}-2}\cdot \nabla u_{j} \nabla \phi>0,~~~\forall
    \phi \in C^{\infty}_{0}(\Omega).
\end{align}
 As a result,  these exists $\phi\geq0$, $\phi\in C^{\infty}_{0}(\Omega)$
such that
\begin{align*}
    \int_{\Omega} |\nabla u_{j}|^{\mathbf{p}-2}\cdot \nabla u_{j}\cdot\nabla \phi \, dx>0.
\end{align*}

Let $\xi>0$ be very small and consider
\begin{align*}
    V_{i}=\begin{cases}
    u_{i}\qquad i\neq  j\\
    [u_{j}-\xi\phi]^{+}  i= j. 
\end{cases}
\end{align*}
From definition of $v_i$ we have  $v_i \cdot v_j = 0$ whenever 
 $i\neq j,$ and  $v_i = \phi_i$  on the boundary of  $\Omega$. 

\begin{align*}
    & J(V)-J(U)=\frac{1}{\bar{\mathbf{p}}}\int_{\Omega}(|\nabla(u_{j}-\xi\phi)^{+}|^{\mathbf{p}}-|\nabla u_{j}|^{\mathbf{p}})dx\\
         &\leq \frac{1}{\mathbf{p}}(\int_{\Omega}|\nabla(u_{j}-\xi\phi)|^{\mathbf{p}}-|\nabla u_{j}|^{\mathbf{p}})\, dx\\
    & \leq -\xi\int_{\Omega}|\nabla u_{j}|^{\mathbf{p}-2}\nabla u_{j}\cdot \nabla\phi \, dx + o(\xi)  \leq 0. 
\end{align*}
Choosing $\xi$ sufficiently small, we obtain
\[
J(V)-J(U)\leq 0.
\]
By assumption, $U$ is the minimizer so $J(U)<J(v)$ is a contradiction with the above inequality. 
Therefore, 
\[
\int_{\Omega}|\nabla u_{j}|^{\mathbf{p}-2}\nabla u_{j}\cdot \nabla\phi \, dx \le 0.
\]

The argument follows the same as the previous part to prove the second part.  Assume  there $i$ such that
\[
 -\Delta_{\mathbf{p}}\hat{u}_{i}\leq 0,
\]
  multiply by test function $\phi\geq 0 $  with compact support at $\Omega_i$. Then
    \[
\int_{\Omega}|\nabla \hat{u}_{i}|^{\mathbf{p}-2}\nabla \hat{u}_{i} \cdot \nabla\phi \, dx \ge 0.
\]
Then since $u_i$ have disjoint supports, from above we get
\[
\int_{\Omega}|\nabla u_{i}|^{\mathbf{p}-2}\nabla u_{i} \cdot \nabla\phi \, dx \le 0.
\]
This contradicts the fact that $ u_i$ is sub $p$ harmonic.

\end{proof}
\begin{remark}
Note that in Lemma \ref{uniqu5} we proved that the differential inequalities hold in the weak sense,  However, the weak and viscosity are equivalent.
\end{remark}
Next, we pass to the limit as $p$ tends to infinity and we obtain that the minimizer for large $p$ satisfies the following
differential inequalities
\begin{enumerate}
  \item $-\Delta_{\mathbf{\infty}}u_{i}\leq 0$
  \item $-\Delta_{\mathbf{ \infty}}\hat{u}_{i}\geq 0$
\end{enumerate}

 \section{Schemes }
 This section presents our algorithms for semi-supervised learning.

\subsection{Infinity Learning Schemes }

Our first method is called Infinity Laplace Learning as follows. We aim to   extend  labels from a discrete set
$ {\{x_i : i = 1, . . . ,l}\}$   to  rest of nodes    $ {\{x_{l+1}, x_{l+2}, \cdots, x_{n} \}}$. Let $k$ denote the number of classes. To indicate that a labeled sample $x_i$ belongs to the $j^{\textit{th}}$ class we write $\uu(x_i)=e_j$ where ${\{e_1, e_2, \cdots ,e_k}\}$ are standard basis for $\mathbb{R}^k$.  Let $\uu=(u_1,\cdots,u_k)$  We solve the following system :
\begin{equation}\label{eq:state13}
\left \{
\begin{array}{ll}
\Delta_{\infty}\uu(x_i) = 0   &  l+1\le i\le n,\\
  \uu(x_i)=y_i \in  {\{e_1, e_2, \cdots e_k}\} &   1\le i \le l,\ \\
  \end{array}
\right.
\end{equation}
Let $\uu=(u_1,\cdots,u_k)$ be a solution of \eqref{eq:state13}, the label of node $x_i\in \cX\setminus\Gamma$ is dictated by
 \begin{equation*}
\underset{j\in {\{1,\cdots, k}\}}{ \textrm{arg max}} {u_{j}(x_i)}.
  \end{equation*}
The System (\ref{eq:state13}) is uncoupled and the existence and uniqueness of solution $u_i$, $i=1, 2,\cdots k,$ follows from \cite{A12} and  $u_i$ is characterized by the fact that it is an absolutely minimal Lipschitz extension of labeling  $y_i$.

\subsection{Infinity Segregated  Learning Schemes }

By Lemma (\ref{uniqu5}) the minimizer satisfies the following differential inequalities
\begin{enumerate}
  \item $-\Delta_{\mathbf{\infty}}u_{i}\leq 0$,
  \item $-\Delta_{\mathbf{ \infty}}\hat{u}_{i}\geq 0$.
\end{enumerate}
Our scheme is based on a segregated system given by Lemma \ref{uniqu5}. The Lemma  \ref{uniqu5} states that each $u_i$ is infinity harmonic in its support, consequently sub-infinity harmonic in the domain,   and $\hat{u}_{i}$ is supper infinity harmonic.
By  using (\ref{L})    to solve infinity Laplacian operator   and  the facts that   $u_{i}(x) \ge 0$ and $u_{i}(x) \cdot u_{j}(x)=0$ and definition of $\hat{u}_{i}$ we obtain the following iterative scheme 
\begin{equation}\label{scheme_sys}
	\begin{cases} 
		u^{(m+1)}_{1}(x_i) =\max \left({u}^{\ast^{(m)}}_1(x_i) - \sum\limits_{p \neq 1}   {u}^{\ast^{(m)}}_p(x_i), \,  0\right),\;\;x_i\in \cX\setminus\Gamma,\\
		u^{(m+1)}_{2}(x_i) =\max \left({u}^{\ast^{(m)}}_{2}(x_i) - \sum\limits_{p \neq 2}  {u}^{*^{(m)}}_p(x_i), \,  0\right),\;\;x_i\in \cX\setminus\Gamma,\\
		\dots\dots\dots\dots\\
		u^{(m+1)}_{k}(x_i) =\max \left( {u}^{\ast^{(m)}}_k(x_i) - \sum\limits_{p \neq k}   {u}^{\ast^{(m)}}_p(x_i), \,  0\right),\;\;x_i\in \cX\setminus\Gamma,\\
		u_{i}(x_i) =\phi_{i}(x),\;\;x_i\in \Gamma,\; \mbox{for all}\; i=1,2,\dots,k.
	\end{cases}
\end{equation}

\section{Implementation}

Considering the results mentioned earlier about $p$-Laplace regularization for large 
$p$, which can lead to solutions that ignore the distribution of the dataset, employing a Siamese Neural Network (SNN) offers a more data-sensitive approach. 
\subsection{SNN\cite{koch2015siamese}} 
 We use SNN to construct graphs over datasets that can be employed to model relationships between data points based on their similarity in a learned feature space. In this setup, the Siamese network helps in learning a suitable similarity measure between data points, which can then be used to create edges in a graph.  The primary role of the SNN is to generate meaningful embeddings for the data points by learning a similarity function.
 
 Here we explain the process: Start with a dataset, where each data point can be represented as a vector (e.g., an image, a feature vector, or any other data type). Next pass pairs of data points through the SNN. The network processes both points and outputs embeddings for each. Since both branches of the network are identical and share the same weights, the embeddings reflect the relative similarity of the input pairs.

After the SNN generates embeddings for each pair of data points,   we use a distance function (like Distance, cosine similarity, etc.) to measure how similar or dissimilar they are. 

\begin{itemize}
    \item \textbf{Distance Metric:} For two data points $x_i$ and $x_j$, the output embeddings from the Siamese network $f(x_i)$ and $f(x_j)$ are compared using a distance function like:

    \[
    \text{Cosine similarity\cite{XIA201539}:} \quad \text{cosine\_similarity}(x_i, x_j) =  \frac{f(x_i) \cdot f(x_j)}{\|f(x_i)\| \, \|f(x_j)\|} 
    \]
    \[
    \text{Adjusted similarity:} \quad \text{sim}(x_i, x_j) = \frac{\text{cosine\_similarity}(x_i, x_j) + 1}{2}
    \]
    \[
    \text{Adjusted distance:} \quad d(x_i, x_j) = \frac{1}{\text{similarity}(x_i, x_j)} - 1
    \]
\end{itemize}
With the similarity scores or nearest neighbors identified, a graph $G = (V, E)$ can be constructed: 
 By training on pairs of data points, the SNN creates embeddings in a feature space where the similarity between data points is directly informed by their relationships in the dataset. Therefore, using an SNN helps ensure that the graph-based models retain important information about the dataset's inherent structure.  This makes SNN an ideal choice for constructing graphs that accurately reflect the data relationships in semi-supervised learning tasks.

\subsection{Results on Different Datasets}

In this subsection, we implement our scheme on some known data sets.  For visualization,

\begin{example}
In our first example, we consider a balanced Two-Moon dataset consisting of 2,000 points with a noise level of 0.15. As the noise increases, the class boundaries become more overlapping. Table  \ref{tab:single_row_table} presents a comparison of the average accuracy between the Infinity Laplacian (InfL), Infinity Segregated Learning (InfSL), and the Poisson scheme, for different numbers of initial labels per class.
 \begin{table}[htbp]
    \centering
    \caption{Two moon, balanced  case}
    \begin{tabular}{|c|c|c|c|c|c|}
        \hline
        \multicolumn{6}{|c|}{\textbf{Average overall accuracy   over 100 trials for two moon}} \\
        \hline
         number of labels per class   &  1 &  2 &  3 &  4 &  5  \\
        \hline
           InfSL    &   \textbf{.9745}   &  \textbf{.9830} &  \textbf{.9811}  &    \textbf{.9856}   &  \textbf{.9861} \\
        \hline
        InfL     &  .8365  &  .8538  & .8775  & .9063  &  .9319 \\
                \hline
        Poisson    &   .7809  & .8523  & .8769  & .9228         &  .9322      \\
         \hline
  \end{tabular}
    \label{tab:single_row_table}
    \end{table}

The results in Table \ref{tab:single_row_table} demonstrate the performance comparison of three algorithms—our proposed InfSL, InfL, and Poisson Learning—on the balanced Two-Moon dataset. InfSL consistently achieves the highest accuracy across all configurations. With just one label per class, InfSL achieves an accuracy of \textbf{97.45\%}, improving to \textbf{98.61\%} with five labels per class. This demonstrates the strength of InfSL in extracting relevant information from minimal labeled data.

\end{example}

\begin{example}
We implement our scheme to the Four-Moons dataset containing 2000 points and a noise level of $0.15$,  with overlapping regions that increase the difficulty for algorithms to classify the points correctly. Figures   \ref{fig:classification_result_1}-\ref{fig:classification_result_5}  show the classification results on the Four-Moons dataset using three methods: InfSL, InfL, and Poisson Learning. Each figure represents a different experiment with varying numbers of initial labels per class.       
 
{    \begin{figure}[htbp]
        \centering
        \includegraphics[width=0.8\textwidth]{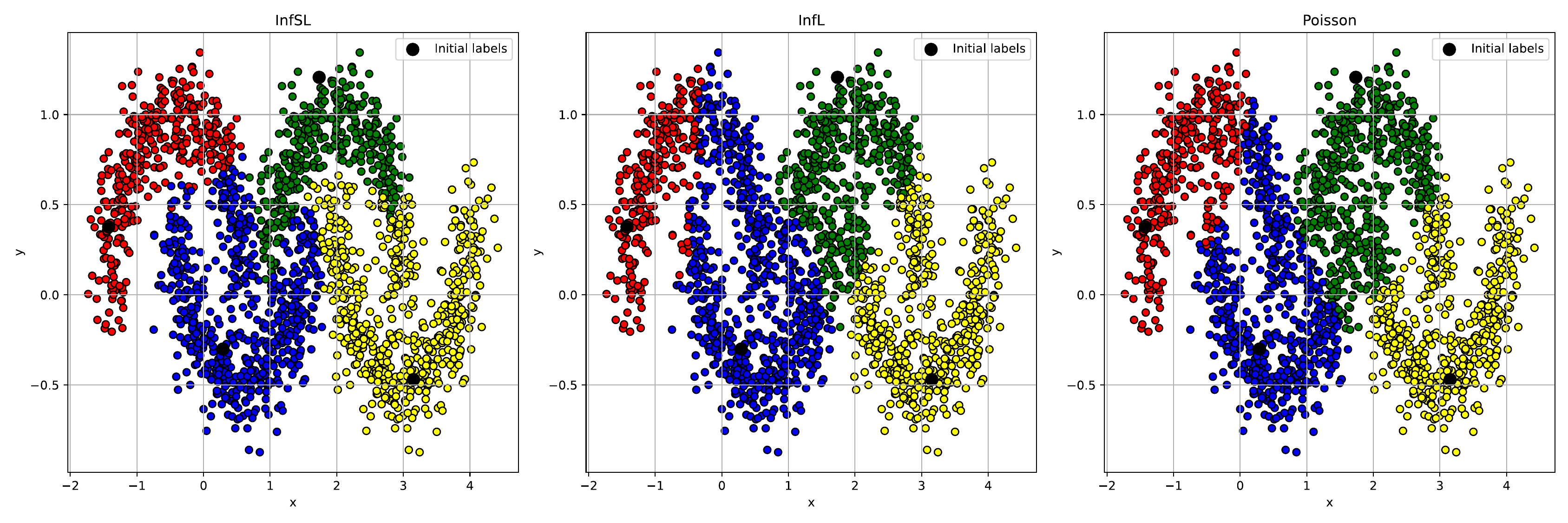}
        \caption{Classification results on \textit{4 Moons} with 1 label per class for InfSL, InfL, and Poisson Learning.}
        \label{fig:classification_result_1}
    \end{figure}    

    \begin{figure}[htbp]
        \centering
        \includegraphics[width=0.8\textwidth]{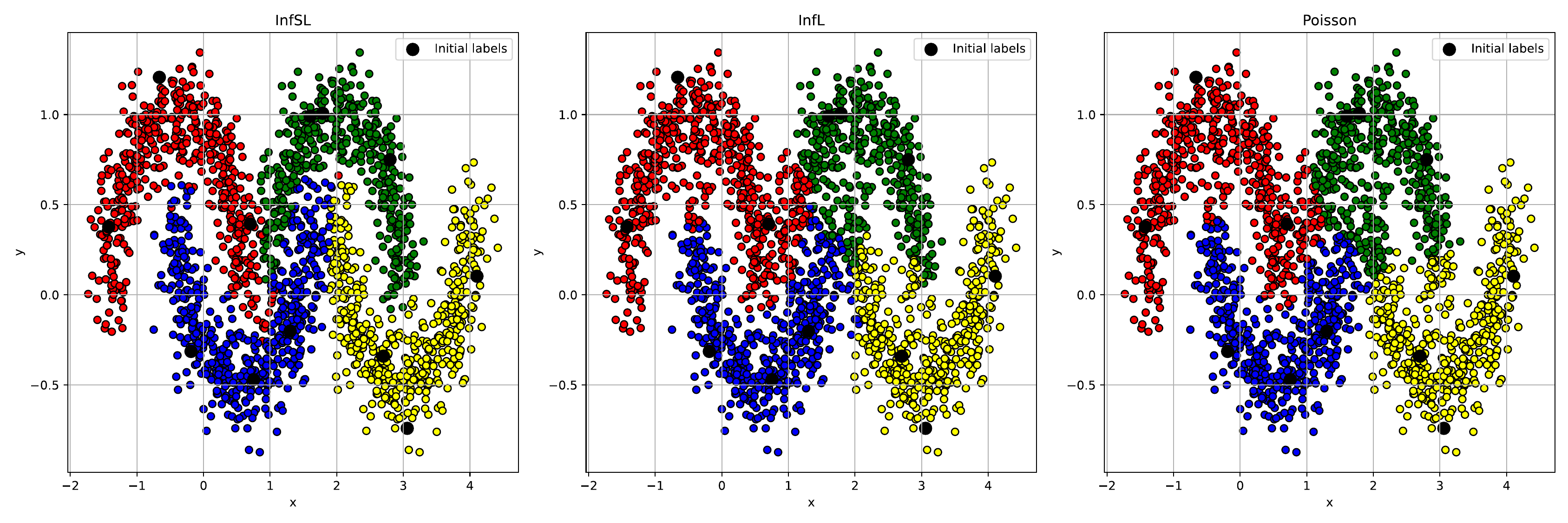}
        \caption{Classification results on \textit{4 Moons} with 3 labels per class for InfSL, InfL, and Poisson Learning.}
        \label{fig:classification_result_3}
    \end{figure}


    \begin{figure}[htbp]
        \centering
        \includegraphics[width=0.8\textwidth]{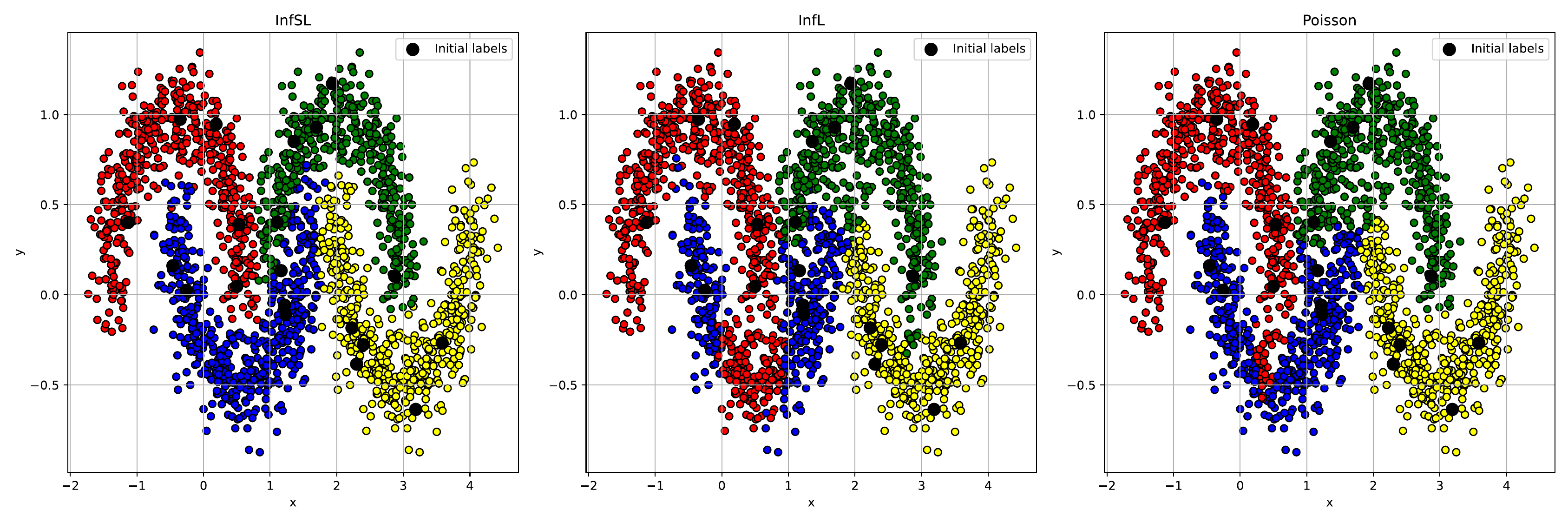}
        \caption{Classification results on \textit{4 Moons} with 5 labels per class for InfSL, InfL, and Poisson Learning.}
        \label{fig:classification_result_5}
    \end{figure}
    }
   In Table \ref{tab:four_moons_results}, we present a comparison of the performance of our proposed method, InfSL, along with InfL and Poisson Learning, on the  Four-Moons dataset. The table displays the average of accuracy for various numbers of initial labels per class.

    \begin{table}[htbp]
        \centering
        \caption{Average accuracy results for the \textit{4 Moons} dataset}
        \begin{tabular}{|c|c|c|c|c|c|}
            \hline
            \multicolumn{6}{|c|}{\textbf{Average overall accuracy over 100 trials for 4 Moons}} \\
            \hline
            Number of labels per class & 1 & 2 & 3 & 4 & 5 \\
            \hline
            InfSL & \textbf{.8215} & \textbf{.8855} & \textbf{.9278} & \textbf{.9367} & \textbf{.9441} \\
            \hline
            InfL & .6562 & .7368 & .7740 & .8226 & .8519 \\
            \hline
            Poisson & .6657 & .7403 & .7894 & .8059 & .8171 \\
            \hline
        \end{tabular}
        \label{tab:four_moons_results}
    \end{table}
The results presented in Table \ref{tab:four_moons_results}
 demonstrate the effectiveness of our proposed method, InfSL, on the more challenging Four Moons dataset. InfSL consistently outperforms both InfL and Poisson Learning across all label configurations. With only one label per class, InfSL achieves an accuracy of \textbf{82.15\%}, significantly surpassing both InfL (\textbf{65.62\%}) and Poisson Learning (\textbf{66.57\%}). As the number of labels per class increases, InfSL continues to outperform, reaching an accuracy of \textbf{94.41\%}  with five labels per class.
    
    These results demonstrate the strength of InfSL in handling complex, non-linear datasets like \textit{4 Moons}, particularly in scenarios with limited labeled data. The performance gap between InfSL and the other methods grows as the classification task becomes more challenging, emphasizing the robustness of InfSL for tasks requiring minimal supervision.

\end{example}

\begin{example}
We evaluated our proposed method on 10 imbalanced benchmark datasets from the KEEL repository, which vary in both the level of class imbalance and sample size \cite{A}, as summarized in Table \ref{Table:datasets}.  In Table. \ref{Table:datasets} the column denoted by  $n$ stands for the total number of samples, $p$ is the number of features, and    IR indicates the ratio between majority and Minority class samples in each data set. 
{\scriptsize{
\begin{table}[!htbp]
  \centering\footnotesize
  \caption{Information of selected imbalanced benchmark datasets}
    \begin{tabular}{l|cccr}    
\hline
\textbf{Dataset} & \textit{IR} & {$p$} & $n$ & \%Minority\\
\hline
    \textbf{ecoli1}	&	3.36	&	7	&	336	&	22.93	\\
    \textbf{ecoli2}	&	5.46	&	7	&	336	&	15.47	\\
    \textbf{ecoli3}	&	8.6	&	7	&	336	&	10.41	\\
    \textbf{shuttle-c0-vs-c4}	&	13.87	&	9	&	1829	&	6.72	\\
    \textbf{glass6}	&	6.38	&	9	&	214	&	13.55	\\
    \textbf{new-thyroid1}	&	5.14	&	5	&	215	&	16.28	\\
    \textbf{new-thyroid2}	&	5.14	&	5	&	215	&	16.28	\\
    \textbf{page-blocks0}	&	8.79	&	10	&	5472	&	10.21	\\
    \textbf{segment0}	&	6.02	&	19	&	2308	&	14.24	\\
    \textbf{vehicle0}	&	3.25	&	18	&	846	&	23.52	\\
    \textbf{vowel0}	&	9.98	&	13	&	988	&	9.10	\\
    \hline
    \end{tabular}%
    \label{Table:datasets}
\end{table}
}}
To evaluate the InfSl Algorithm, we use the metrics   \textit{${\rm F}_1$}-Score,   Recall,  Accuracy, and Precision for each class. To ensure consistency for all experiments, the data set is first shuffled for each benchmark. Subsequently, 1 percent of the samples are randomly chosen in accordance with the dataset's IR   as the labeled samples. This process is independently repeated 100 times, then the averages of the aforementioned metrics are computed. 

     Table \ref{4} outlines the detailed performance metrics, highlighting consistent improvements in accuracy, F1-score, recall, and precision across the imbalanced datasets. The results demonstrate that our method is particularly effective in handling extreme class imbalances, as seen with the "shuttle-c0-vs-c4" dataset.   
 {\scriptsize{
 \begin{table}[!htbp]
  \centering\footnotesize
  \caption{Information of selected imbalanced benchmark datasets}
  \begin{tabular}{|c|c|c|c|c|c|c|c|c|c|c|}

    \hline

     Algorithm &  \multicolumn{7}{|c|}{InfSL} \\ 

    \hline

    Dataset &  Accuracy & F1 min & F1 maj & Recall min & Recall maj & Precision min & Precision maj\\ %
    \hline
    ecoli1    &   $0.8836$   &   $0.6974$   &   $0.9274$   &   $0.8352$   &   $0.8996$   &   $0.6287$   &   $0.9593$ \\
    ecoli2    &   $0.9352$   &   $0.8168$   &   $0.9583$   &   $0.8518$   &   $0.9644$   &   $0.8115$   &   $0.9579$ \\
    ecoli3    &   $0.9556$   &   $0.7918$   &   $0.9749$   &   $0.8379$   &   $0.9743$   &   $0.7757$   &   $0.9765$ \\
    shuttle-c0-vs-c4    &   $0.9895$   &   $0.904$   &   $0.9944$   &   $1.0$   &   $0.989$   &   $0.8433$   &   $1.0$ \\
    glass6    &   $0.9693$   &   $0.9377$   &   $0.978$   &   $0.9492$   &   $0.994$   &   $0.9607$   &   $0.9706$ \\
    new-thyroid1    &   $0.9908$   &   $0.9725$   &   $0.9945$   &   $0.947$   &   $0.9999$   &   $0.9994$   &   $0.9891$ \\
    new-thyroid2    &   $0.9862$   &   $0.9593$   &   $0.9917$   &   $0.9218$   &   $1.0$   &   $1.0$   &   $0.9835$ \\
    page-blocks0    &   $0.966$   &   $0.7991$   &   $0.9814$   &   $0.9805$   &   $0.965$   &   $0.6806$   &   $0.9985$ \\
    segment0    &   $0.9652$   &   $0.8489$   &   $0.9803$   &   $0.9962$   &   $0.962$   &   $0.758$   &   $0.9996$ \\
    vehicle0    &   $0.9905$   &   $0.9801$   &   $0.9938$   &   $0.9709$   &   $0.9967$   &   $0.9894$   &   $0.9909$ \\
    vowel0    &   $0.981$   &   $0.8809$   &   $0.9897$   &   $1.0$   &   $0.9796$   &   $0.7916$   &   $1.0$ \\
    \hline
    \end{tabular}
 \label{4}
\end{table}
}}
\end{example}
 
\begin{example}   

Next, we evaluated the performance of our method and compared it with InfL and the Poisson scheme using the well-known \textit{MNIST} dataset. Table \ref{tab:mnist_results} presents the average accuracy results over 100 trials, with varying numbers of labeled points per class (ranging from 2 to 10).

\begin{table}[htbp]
    \centering
    \caption{Accuracy Results for the \textit{MNIST} Dataset}
    \resizebox{\textwidth}{!}{%
    \begin{tabular}{|c|c|c|c|c|c|c|c|c|c|}
        \hline
        \multicolumn{10}{|c|}{\textbf{Average overall accuracy over 100 trials for \textit{MNIST} Dataset}} \\
        \hline
        number of labels per class  &  2 &  3 &  4 &  5 & 6 & 7 & 8 & 9 & 10 \\
        \hline
         InfSL    &   \textbf{.994606}   &  \textbf{.994663} &  \textbf{.994672}  &    \textbf{.994678}   &  \textbf{.994677} & \textbf{.994682} & \textbf{.994689} &  \textbf{.994690} & \textbf{.994690} \\
        \hline
        InfL     &  .797549  &  .850493  & .880603  &  .899981 &  .913188  &    .924895 &    .928351   &   .934989    &    .939326 \\
         \hline
         Poisson     &  .781556  &  .849058  & .885421  &  .908756 &  .915682  &    .92546 &    .929523   &   .936487    &    .940658 \\
         \hline
    \end{tabular}%
    }
    \label{tab:mnist_results}
\end{table}


Table \ref{tab:mnist_results} provides a  comparison of the performance of InfSL, InfL, and Poisson Learning on the MNIST dataset. The rows display the average accuracy achieved for each method, while the columns indicate the number of labeled samples per class, ranging from 2 to 10.

 InfL and Poisson Learning start with lower accuracies, particularly with 2 labeled samples per class, but gradually improve as more labeled samples are provided. By 10 labeled samples per class, Poisson Learning slightly surpasses InfL, though both remain below the performance of InfSL. This demonstrates the advantage of the graph refinement process in InfSL, which effectively captures the underlying structure of the data with minimal supervision.

\end{example}

\begin{example}
Next, we evaluated the performance of our method and compared it with Infinity Learning (InfL) and the Poisson scheme using a real-world dataset collected in collaboration with medical professionals. This dataset comprises a total of 452 images, including 205 images of Koilocytotic—cells that exhibit specific morphological changes indicative of cervical cancer and 247 images of normal cells. Figure \ref{fig:cell_examples} shows representative examples of the two classes in the dataset, with 9 images for each class. The inclusion of these diverse images allows for a comprehensive evaluation of our method’s ability to differentiate between Koilocytotic and normal cells, which is critical for effective cervical cancer screening.

\begin{figure}[htbp]
  \centering
  \begin{minipage}{0.4\textwidth}
    \centering
    \includegraphics[width=0.26\textwidth, height=1.5cm]{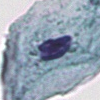}\quad
    \includegraphics[width=0.26\textwidth, height=1.5cm]{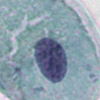}\quad
    \includegraphics[width=0.26\textwidth, height=1.5cm]{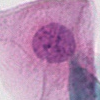}\\
    \includegraphics[width=0.26\textwidth, height=1.5cm]{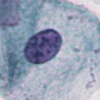}\quad
    \includegraphics[width=0.26\textwidth, height=1.5cm]{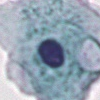}\quad
    \includegraphics[width=0.26\textwidth, height=1.5cm]{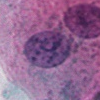}\\
    \includegraphics[width=0.26\textwidth, height=1.5cm]{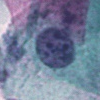}\quad
    \includegraphics[width=0.26\textwidth, height=1.5cm]{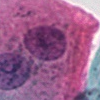}\quad
    \includegraphics[width=0.26\textwidth, height=1.5cm]{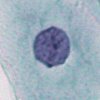}
    \caption{(a) Normal}
    \label{fig:normal_cells}
  \end{minipage}
  \hspace{0.1\textwidth}
  \begin{minipage}{0.4\textwidth}
    \centering
    \includegraphics[width=0.26\textwidth, height=1.5cm]{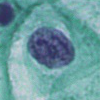}\quad
    \includegraphics[width=0.26\textwidth, height=1.5cm]{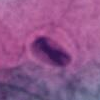}\quad
    \includegraphics[width=0.26\textwidth, height=1.5cm]{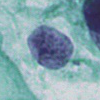}\\
    \includegraphics[width=0.26\textwidth, height=1.5cm]{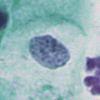}\quad
    \includegraphics[width=0.26\textwidth, height=1.5cm]{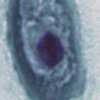}\quad
    \includegraphics[width=0.26\textwidth, height=1.5cm]{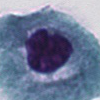}\\
    \includegraphics[width=0.26\textwidth, height=1.5cm]{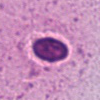}\quad
    \includegraphics[width=0.26\textwidth, height=1.5cm]{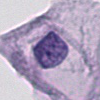}\quad
    \includegraphics[width=0.26\textwidth, height=1.5cm]{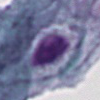}
    \caption{(b) Koilocytotic}
    \label{fig:koilocytotic_cells}
  \end{minipage}
  
  \caption{Representative examples of cells from the normal (a) and koilocytotic (b) classes in the real dataset.}
  \label{fig:cell_examples}
\end{figure}

Table \ref{tab:real_data_results} provides a comparison of the performance of InfSL, InfL, and Poisson Learning on the real dataset. The rows display the average accuracy achieved for each method, while the columns indicate the number of labeled samples per class, ranging from 2 to 7.

\begin{table}[htbp]
    \centering
    \caption{Accuracy Results for the \textit{Real} Dataset}
    \resizebox{\textwidth}{!}{%
    \begin{tabular}{|c|c|c|c|c|c|c|}
        \hline
        \multicolumn{7}{|c|}{\textbf{Average overall accuracy over 100 trials for the \textit{Real} Dataset}} \\
        \hline
        number of labels per class  &  2 &  3 &  4 &  5 & 6 & 7 \\
        \hline
         InfSL    &   .7435   &  .8044 &  .8707  &    .8955   &  .9188 & .9295  \\
        \hline
        InfL     &  \textbf{.7501}  &  \textbf{.8330}  & \textbf{.8840}  &  \textbf{.9100} &  \textbf{.9438}  &    \textbf{.9503} \\
         \hline
         Poisson     &  .7128  &  .7815  & .8041  &  .8088 &  .8138  &    .8227 \\
         \hline
    \end{tabular}%
    }
    \label{tab:real_data_results}
\end{table}




These results highlight the effectiveness of InfSL in medical image classification tasks, particularly for identifying Koilocytotic associated with cervical cancer. While InfL consistently outperforms InfSL, the latter still demonstrates strong capabilities, especially as the amount of labeled data increases. This underscores the potential of InfSL in practical applications for medical diagnostics, providing valuable support for detecting critical conditions like cervical cancer.

\end{example}

\section{Conclusion}
In this work, we have developed an efficient and robust graph-based semi-supervised learning framework leveraging the Infinity Laplacian operator. By extending spatial segregation theory to the graph-based Infinity Laplacian, the proposed Infinity Segregated Learning (InfSL) method demonstrates exceptional performance, particularly in handling imbalanced datasets and scenarios with sparse labeling.  We Compared various datasets to establish InfSL's effectiveness over existing methods like Poisson Learning and standard Infinity Laplace schemes, highlighting its superior classification accuracy and resilience to class imbalance. This study advances semi-supervised learning by offering a novel method that effectively balances label propagation with dataset structure. It is highly suitable for real-world applications in fields requiring minimal labeled data, such as medical diagnostics and image processing.


\bibliographystyle{abbrv}
\bibliography{references}

\end{document}